\documentclass[letterpaper]{article} 
\usepackage[preprint]{aaai2027}  
\usepackage[hyphens]{url}  
\usepackage{graphicx} 
\usepackage{natbib}  
\usepackage{caption} 
\usepackage{algorithm}
\usepackage{algorithmic}

\usepackage{newfloat}
\usepackage{listings}
\DeclareCaptionStyle{ruled}{labelfont=normalfont,labelsep=colon,strut=off} 
\floatstyle{ruled}
\newfloat{listing}{tb}{lst}{}
\floatname{listing}{Listing}

\usepackage{booktabs}

\usepackage{graphicx} 
\usepackage{mathtools}
\usepackage{amsmath}
\usepackage{amssymb}
\usepackage{amsthm}
\usepackage{makecell}

\theoremstyle{definition}
\newtheorem{theorem}{Theorem}
\newtheorem{definition}{Definition}
\newtheorem{lemma}{Lemma}

\title{AI Safety Considerations for Agents With Limited Time to Act}
\author{
    Leo Zeitler \textsuperscript{\rm 1} \corresponding,
    Jack Richings \textsuperscript{\rm 1}, 
    Victoria Nockles \textsuperscript{\rm 1}
}

\affiliations{
    \textsuperscript{\rm 1}The Alan Turing Institute\\

    96 Euston Road \\ 
    London, NW1 2DB, United Kingdom\\
    \corresponding lzeitler@turing.ac.uk
}

\begin{document}
\maketitle

\begin{abstract}
In the wake of the increasingly public discussion about AI alignment, recent work has tried to propose specific AI architectures that behave safely. However, the proposed arguments that seemingly demonstrate proved alignment mostly neglect the environment the agent needs to act in. We discuss theoretical bounds for agent-agnostic safety guarantees in environments that can only be partially observed and within which an action is required within limited time. We introduce two realistic scenarios, one with an infinite state space and one with signal mixture. In these scenarios, we prove that even a perfect agent cannot guarantee safe behaviour. It will be argued that for any proof of AI safety or alignment, the environment and associated safe actions need to be specifically considered together with the agent. 
\end{abstract}

\section{Introduction}
Providing safe AI solutions is pivotal when using autonomous agents in many domains, for example defence, but proving safety is notoriously difficult. The impossibility of general safety guarantees has been previously discussed and proved \cite{rice1953classes, de2024undecidability, yao2025alignment, nourizadeh2025no}. Nevertheless, there is the overall aim to make AI safer and to provide formal verification of its use. \textit{Guaranteed safe AI}, which includes a probabilistic world model, safety-specifications, and a verifier, was proposed as a versatile and safe AI system \cite{dalrymple2024towards}. Even critics that proved impossibility of guaranteed general alignment claimed that alignment can be achieved when considering a specific model design or architecture \cite{de2024undecidability, nourizadeh2025no}. \cite{tegmark2023provably} argued that provably safe artificial intelligence will become possible in the future with automated software verification. To date, most work has focused on the agent itself, and real-world conditions which require an action with incomplete information and limited time are typically ignored. Most literature is concerned with the \textit{decidability} problem of the agent's alignment rather than \textit{identifiability} of the true environmental state, on which, ultimately, the alignment is based.

\par A relevant example is an agent at a dispatch center which needs to discard parcels with dangerous content after opening them based on a camera stream. The center also supplies a hospital which requires fast delivery of new medication. Any delay risks the death of patients in the hospital, and the system needs to take action before it can ask for human validation. Upon detecting a glass vial with a transparent liquid, it is impossible to determine whether the substance is a malicious toxin or a life-saving drug based on the camera stream. The safe actions that should be performed in either state are mutually distinct. AI safety cannot be guaranteed, independent of the model used for the agent.

\begin{figure*}[!htp]
    \centering
    \includegraphics[width=.9\linewidth]{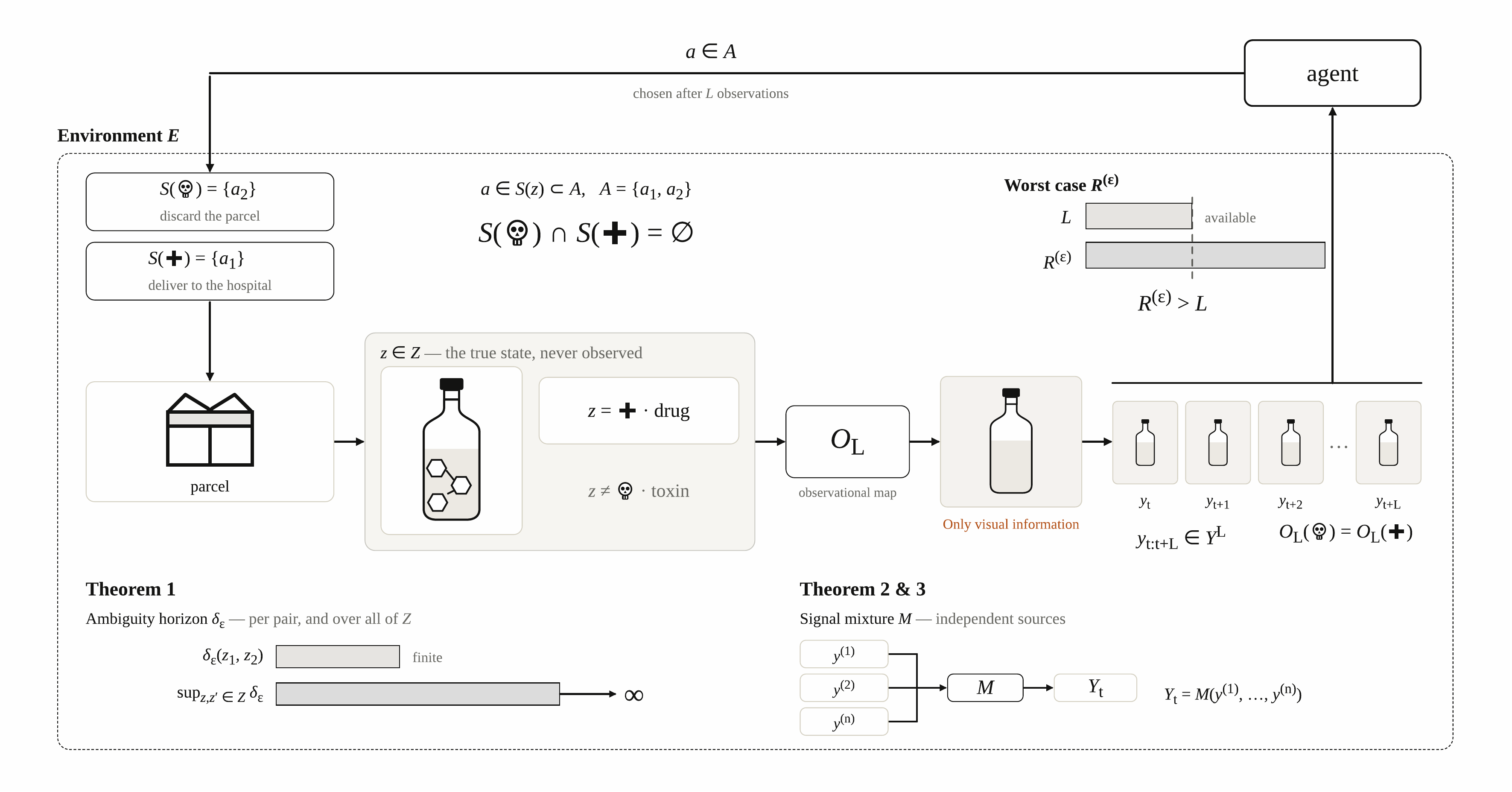}
    \caption{\textbf{Two environmental states can become observationally indistinguishable under real-world conditions.} This schema provides intuition of bounds to AI-safety claims that are independent of the used model. Consider an agent that needs to forward life-saving medicine to a hospital but is also required to discard hazardous items. The action is decided using image analysis taken from a video stream. When it finds a vial with a transparent liquid, the environmental state in which the liquid is a dangerous toxin is observationally indistinguishable from a liquid that is used for medication. The safe action in both states are mutually distinct, and no autonomous agent can guarantee safety.}
    \label{fig:hazzard-schema}
\end{figure*}

\par Here, we investigate how the action space and the environment must be structured to allow an autonomous controller to behave safely. The agent is seen as independent of the environment as well as the safe actions and the observational model, both of which are constraints imposed by the engineer to the controller. By drawing on Gold's language identification in the limit \cite{gold1967language}, the theory of infinite games \cite{fridman2011degrees, holtmann2012degrees, zimmermann2015unbounded}, and the framework of partially observable Markov decision problems (POMDPs) \cite{kaelbling1998planning}, this paper introduces a formal framework and discusses the problem of identifying, within limited time, the true environmental state, which is distinct from the agent's internal representation, to perform an adequate safe action. Superposition of signal sources, referred to in the field of voice recognition as the cocktail party problem \cite{cherry1953some}, will be specifically considered. Arguing that realistic scenarios can give rise to superimposed observations, we prove that distinct environmental states, which were identifiable when using independently-observable signal sources, can become non-identifiable upon signal mixture, regardless of the autonomous agent. It follows that no AI system can guarantee safety without considering explicitly the environment, the observations that can be made, and the safe action space. AI alignment and safety studies that ignore these properties are incomplete. Rather than claiming completeness, the paper aims to provide formal examples where AI safety requirements are independent of the applied agent. Our contributions are as follows:

\begin{itemize}
    \item We provide formal and system-agnostic definitions for evaluating AI safety and associated criteria for partially observable environments in which the agent has limited time to act;
    \item We prove that empirical, finite testing of safety requirements is insufficient for environments with infinite state space, and we provide two conditions that can be tested to provide a bound;
    \item We show that even in environments with finite state space, observational mixture of signals, such as superposition, lead to indistinguishability of environmental states, even if they are perfectly distinguishable without signal mixture.
\end{itemize}

An overview over the formalism is given in Table \ref{table:abbr}.

\begin{table*}[!ht]
\renewcommand{\arraystretch}{1.5}
\begin{center}
\begin{tabular}{ l | l }
\toprule
 \textbf{Symbol} & \textbf{Explanation} \\ 
 \hline
 $E$ & Partially-observable environment \\ 
  $L$ & \makecell[l]{Latency, number of observations before an action \\ is required. It is a quantity imposed by the engineer or the environment}  \\
  $a \in \mathcal{A}$ & Action space \\
  $y_{t:t+L} \in \mathcal{Y}^{L + 1}$ & Observations made during latency. Per convention, $t=0$ \\
  $z\in \mathcal{Z}$ & Environmental state space \\
  $O_L$ & \makecell[l]{Observational map, \\ observations that can be made from an environmental state} \\
  $a \in \mathcal{S}(z) \subset \mathcal{A}$ & Safe actions associated to an environmental state \\
  $\delta$ & \makecell[l]{Ambiguity horizon, the number of observations \\ needed to distinguish any two environmental states} \\
  $\mathcal{V}_k$ & \makecell[l]{Set of environmental states that \\ that are indistinguishable with $k$ observations} \\
  $\hat{\mathcal{V}}_k$ & \makecell[l]{Agent's representation of the set of environmental states \\ that are indistinguishable with $k$ observations} \\ 
  $y_{\infty} \in \mathcal{Y}^{\infty}$ & \makecell[l]{An infinitely long sequence of observations.}\\ 
  $r(y_\infty)$ & \makecell[l]{The minimum number of observations at which \\a common safe action for all possible latent states exists.}\\
  $R$ & \makecell[l]{Worst-case scenario of the number of observations \\ necessary before the agent can identify a safe action} \\
  $Y_t$ & Vector of observations from independent signal sources \\
  $M$ & Mixture operator that causes superposition. \\
  \bottomrule
\end{tabular}
\end{center}
\caption{Mathematical symbols and concepts used throughout the paper and ordered with respect to their appearance.}
\label{table:abbr}
\end{table*}

\section{Results}
Let $E$ denote the environment in which an artificial controller acts. Then, define a controller as latency-$L$ if its action $a_t \in \mathcal{A}$ at time point $t$ is a function $\pi_L\left(y_{t:t+L}\right)$ of observations $y_{t:t+L} \in \mathcal{Y}^{L + 1}$, meaning that the system can make maximally $L$ more observations before taking an action. $\mathcal{Y}^{L + 1}$ contains the immediate observation at $t$ and $L$ additional observations. Time is quantized, such that there is an observation at each time point. In the following, we only discuss results regarding $L$, and, without loss of generality, we will set $t=0$, which will be omitted from the notation where applicable. It is assumed that $E$ is in a state $z \in \mathcal{Z}$ that can only be partially observed by the controller and is therefore hidden. Let $O_L: \mathcal{Z} \longrightarrow \mathcal{Y}^{L + 1}$ be the observational map, such that $O_L(z) = y_{0:L}$ are the observations for a latent state $z$ which are perceived by the agent before selecting an action $a$. $O_L$ defines the observable part of the environment, i.e. the input to the agent. In the following, we will consider the theoretical construct of a perfect agent, meaning that the system always chooses the optimal action $a$ in all cases as soon as it can identify the underlying environmental state $z$. Our central claim is that environments and possible observations need to be specifically considered in any AI safety proof, even when the agent is optimal. If there are $z, z' \in \mathcal{Z}$ with $O_L(z) = O_L(z')$ while the required actions are $a(z) \neq a(z')$, then no deterministic latency-$L$ controller can be correct on both $z$ and $z'$. We will provide examples for infinite and finite environmental state spaces $\mathcal{Z}$ for which AI safety cannot be guaranteed, demonstrating why AI safety measures that ignore environment and observations are incomplete. 

\par This can be easily extended to a safety framework. Let $\mathcal{S}(z) \subset \mathcal{A}$ be a set of safe actions when the environment is in state $z$. 
If two latent states $z, z'$ become observationally indistinguishable with $O_L(z) = O_L(z')$ but $\mathcal{S}(z) \cap \mathcal{S}(z') = \emptyset$, then no latency-$L$ controller can guarantee safety for both $z$ and $z'$.

\par We claim that observational indistinguishability exists in open environments, which are here defined as settings that are only partially observable by the agent. $L$ governs the tradeoff between reliability and performance. $L$ should be minimised for the latency-$L$ controller to be efficiently operational in the environment, but it must be large enough to minimise the occurrence of observational indistinguishability. The value of $L$ is an engineering choice which requires a measurable quantity.

\par The ambiguity distance between $z, z' \in \mathcal{Z}$, $z \neq z'$ will be defined as $\delta(z, z') = \inf \{\lambda | O_{\lambda}(z) \neq O_{\lambda}(z')\}$, i.e. the latency $\lambda$ after which $z$ becomes distinguishable from $z'$. Moreover, $\delta(z, z) = 0$, meaning that the ambiguity distance with a state itself is 0, and $\delta(z, z') = 0$ if  $O_{0}(z) \neq O_{0}(z')$, i.e. the ambiguity distance is 0 for states that are immediately distinguishable. We highlight that we only discuss deterministic environments and observational maps, and a single observation can suffice to make to states distinguishable.

\par This allows formulating the following objective for smart systems with safety requirements. Introduce $\mathcal{V}_{\lambda}(y_{0:\lambda}) = \{z \in \mathcal{Z} | O_\lambda(z) = y_{0:\lambda}\}$ the set of all environmental states that lead to the same $\lambda$ observations. Note that this property is fundamentally limited by the observational map $O_L$, and it is independent of which AI model is used. Training the agent aims to find an internal representation $\hat{\mathcal{Z}}$ for which the equivalently defined set $\hat{\mathcal{V}}_{\lambda}(y_{0: \lambda}) = \{\hat{z} \in \hat{\mathcal{Z}} | O_\lambda(\hat{z}) = y_{0:\lambda}\}$ collapses in a sequence as quickly as possible with $\hat{\mathcal{V}}_0 \supseteq \hat{\mathcal{V}}_1 \supseteq \hat{\mathcal{V}}_2 ...\supseteq \hat{\mathcal{V}}_L \supset \emptyset$. Note that this is bounded by the true environmental state sequence $\mathcal{V}_k$, which is approached by $\hat{\mathcal{V}}_k$. Moreover we assume that both $\mathcal{V}_k$ and $\hat{\mathcal{V}}_k$ are never empty, meaning that there exists always a state $z$ that can produce the observations. Consequently, for an environment $E$, an observation map $O$, and safe action $\mathcal{S}$, guaranteed safety at latency-$L$ requires $\bigcap_{z \in \mathcal{V}_L} \mathcal{S}(z) \neq \emptyset$ over all observations.

\par The safety of an AI system is governed by the worst-case scenario. Let $y_{\infty} \in \mathcal{Y}^{\infty}$ be an infinitely long observation. We define the auxiliary function $r(y_{\infty}) = \inf\{\lambda | \bigcap_{z \in V_\lambda(y_{0:\lambda})} S(z) \neq \emptyset\}$, with $r(y_{\infty}) = 0$ if $\bigcap_{z \in V_0} S(z) \neq \emptyset$ and $r(y_{\infty}) = \infty$ for $\inf \emptyset$. $r$ is the minimum number of observation at which a common safe action exists. Note that $V_{\lambda + 1} \subseteq V_{\lambda}$, and consequently, if $r(y) = \lambda$, then $\bigcap_{z \in V_{\lambda'}} \mathcal{S}(z) \neq \emptyset$ for $\lambda' > \lambda$. Then, the worst-case scenario is bound by supremum of $r$ over all possible $y_{\infty}$

\begin{equation*}
    R = \sup_{y_{\infty}} r(y_{\infty})
\end{equation*}

which will be referred to as $R$-safety. $R$ is the longest time it takes to distinguish states with disjoint safe action space. Note that $R$ is more dependent on the observation model, rather than the type or implementation of the agent. Three cases can be identified.
\newline

\par \textbf{Safe:} When $R \leq L < \infty$, $R$-safety can be assured.
\par \textbf{Application-incompatible:} When $L_{available} < R < \infty$, guaranteed $R$-safety is impossible but can be gauged. $z$ is delayed identifiable.
\par \textbf{Unbound:} When $R = \infty$, no finite time horizon can guarantee $R$-safe behaviour. Note that this does not imply that there must be an individual state $z$ which indistinguishable. Every pair $z$ might have a finite ambiguity horizon $\delta$, but the supremum over an infinite set $\mathcal{V}_\infty$ can be infinite.  
\newline

In the following, we demonstrate that $\delta$ and $R$ are related. In fact, the pairwise ambiguity horizon $\delta$ provides a lower bound for $R$.

\begin{lemma}
Let $z, z'\in \mathcal{Z}$ with $z \neq z'$. Then
$R \;\ge\; \sup\{\delta(z,z')\mid z\neq z',\ S(z)\cap S(z')=\emptyset\}$.
\label{lemma:deltar}
\end{lemma}

\begin{proof}
    Let $z\neq z'$ with $S(z)\cap S(z')=\emptyset$. Moreover, set $d=\delta(z,z')$, and let $y_\infty=O_\infty(z)$. We distinguish three cases
    \begin{itemize}
        \item $d = \infty$: $O_\lambda(z) = O_\lambda(z')$ for all $\lambda \in \mathbb{N}$. Consequently, $z, z' \in \mathcal{V}_\lambda$ and $\bigcap_{z'' \in \mathcal{V}_\lambda} \mathcal{S}(z'') \subseteq \mathcal{S}(z) \cap \mathcal{S}(z') = \emptyset$ for all $\lambda \in \mathbb{N}$.
        \item $0 < d < \infty$: since $\delta(z, z') = d$ is defined as the infimum, $d-1\notin\{\lambda \mid O_\lambda(z)\neq O_\lambda(z')\}$. It follows that $O_{d-1}(z)=O_{d-1}(z')=y_{0:d-1}$. Then $z,z'\in V_{d-1}(y_{0:d-1})$, and as above $\bigcap_{z''\in V_{d-1}}S(z'')=\emptyset.$ $d-1$ is in the set defining $r(y_\infty)$, and consequently, $R\ge r(y_\infty) > d-1$.
        \item $d = 0$: $R \geq d = 0$ holds by convention.
    \end{itemize}

    Since this holds for every pair $z, z' \in \mathcal{Z}$ with $\mathcal{S}(z) \cap \mathcal{S}(z') = \emptyset$, taking the supremum over all such pairs gives the claim.
\end{proof}

This gives rise to the question of how safe actions need to be designed specific to the environment to allow $R$-safety.

\subsection{Infinite state space can lead to unbounded delay}
It is assumed that many real-world applications for autonomous agents include an infinite number of environmental states, i.e. $|\mathcal{Z}| = \infty$. This section describes properties of the safe action space to allow $R$-safety by demonstrating that there exist constructions of the maximum ambiguity $\delta$, which is dependent on the observational map, for which the maximum over the state space $\mathcal{Z}$ is unbound. We assume
$\mathcal{S}(z)\neq\emptyset$ for all $z\in\mathcal{Z}$. For this, we define an ordered sequence called ambiguity sequence.

\begin{definition}
Let $\mathcal{P}=\{\{z,z'\}\subseteq\mathcal{Z} \mid z\neq z'\}$ be the set of unordered pairs of distinct environmental states, and write $\delta(p)=\delta(z,z')$ for $p=\{z,z'\}\in\mathcal{P}$. Define the strict order $\gamma$ on $\mathcal{P}$ by
\begin{equation*}
  \gamma(p,q) \iff \delta(p)<\delta(q).
\end{equation*}

A sequence $(p_k)_{k\in\mathbb{N}}\subseteq\mathcal{P}$ is an
\textit{ambiguity sequence} if $\gamma(p_k,p_{k+1})$ for all $k\in\mathbb{N}$, i.e.\
$\phi: p_k\mapsto k$ is an order isomorphism
$(\{p_k\}_{k\in\mathbb{N}},\gamma)\to(\mathbb{N},<)$. Moreover, a pair $p=\{z,z'\}$ is a \emph{conflict pair} if
$\mathcal{S}(z)\cap\mathcal{S}(z')=\emptyset$, and
$\mathcal{P}_\mathcal{S}\subseteq\mathcal{P}$ denotes the set of conflict pairs.
\end{definition}
 
\begin{theorem}
Let $\mathcal{Z}$ be infinite and $\delta(p)<\infty$ for all $p\in\mathcal{P}$.
Then
\begin{enumerate}
  \item $\sup_{p\in\mathcal{P}}\delta(p)=\infty$ if and only if $\mathcal{P}$ contains an ambiguity sequence;
  \item if $\mathcal{P}_\mathcal{S}$ contains an ambiguity sequence, then $R=\infty$, although every pair has a finite ambiguity distance;
  \item if $\mathcal{P}$ contains no ambiguity sequence, then $R\le\sup_{p\in\mathcal{P}}\delta(p)<\infty$.
\end{enumerate}
\label{theorem:infinite}
\end{theorem}

\begin{proof}
To 1: If $(p_k)_{k\in\mathbb{N}}$ is an ambiguity sequence, the values $\delta(p_k)$ form a strictly increasing sequence of natural numbers, and there exists an isomorphism $\phi$ such that $\phi(\gamma, (p_k, p_{k+1})) = k < k + 1$ for all $k$. It follows that $\sup_{p\in\mathcal{P}}\delta(p)=\infty$. Conversely, let $\sup_{p\in\mathcal{P}}\delta(p)=\infty$ with all values finite. Choose $p_0\in\mathcal{P}$ arbitrarily. Given $p_k$ with $k \geq 0$, it follows from $\sup_{p\in\mathcal{P}}\delta(p)=\infty$ that there exists $p_{k+1}\in\mathcal{P}$ with $\delta(p_{k+1})>\delta(p_k)$, which inductively constructs an ambiguity sequence.
\newline

To 2: As proved above, an ambiguity sequence $\mathcal{P}_\mathcal{S}$ implies $\sup_{p\in\mathcal{P}_\mathcal{S}}\delta(p)=\infty$. Together with  Lemma \ref{lemma:deltar}, it follows that $R=\infty$.
\newline

To 3: It follows from 1 that $D=\sup_{p\in\mathcal{P}}\delta(p)<\infty$. Let $y_\infty\in\mathcal{Y}^\infty$
and $\lambda$ be such that
$\bigcap_{z\in V_\lambda(y_{0:\lambda})}\mathcal{S}(z)=\emptyset$. As
$\mathcal{S}(z)\neq\emptyset$ for all $z$, $V_\lambda(y_{0:\lambda})$ contains two
distinct states $z,z'$. Then $O_\lambda(z)=O_\lambda(z')$ and
$\delta(z,z')\ge\lambda+1$, and since no ambiguity sequence exists, $\lambda\le D$, so $r(y_\infty)\le D$ for all $y_\infty$, and it follows that $R\le D$.
\end{proof}

Any empirical testing can only evaluate a finite number of environmental states, which corresponds to the set $\hat{\mathcal{Z}}$. $R$-safety can only be formally shown if $R$, which is bounded by the ambiguity distance $\delta$ (Lemma \ref{lemma:deltar}) of environmental states with mutually distinct safety actions, does not grow with $\delta$. However, as shown in Theorem \ref{theorem:infinite}, any testing which is based on finite sampling can provide misleading upper bounds for $\delta$. Any attempt to prove $R$-safety for bounded $R$ can only be provided if one of the following is also shown:

\begin{enumerate}
    \item The set of sets $\{\mathcal{V}_L: \bigcap_{z \in \mathcal{V}_L}\mathcal{S}(z) = \emptyset \}$ is finite. This does not proof $R$-safety, but $R$ can be explicitly computed.
    \item The injective map $\phi$ does not exist, and consequently, there is no ambiguity sequence, provided all pairwise $\delta$ are finite.
\end{enumerate}

Note that even if there does not exist an injective map $\phi$ you might not be able to compute $\delta$ explicitly, yet it is guaranteed that $R$ will be bounded. An example for a system for which $\phi$ exists are environments where the observation itself can only be made after a lapsed time $T$, as it is the case for sound or light. Here, $T$ increases as a function of distance, and two states $z_m$ and $z_n$ can only be distinguished after the signal arrives, which can, if the distance can be arbitrarily long, go to infinity.   

\subsection{Considerations regarding signal superposition in the environment}
Special attention is required for open environments that are also subject to superpositions of observations \cite{cherry1953some}. This can render latent states $z$ non-identifiable, as proved below. Importantly, only the case of \textit{observational signal mixture} is considered here, and the superposition of \textit{latent environmental states} $z$ and $z'$ is ignored. Let there be several signal sources $s_u$ that give rise to observations $y_i$. Let $Y_t = \left(y_t^{(1)}, ..., y_t^{(k)}\right)$ be the vector of $k$ clean observations from $k$ signal sources $s_k$. Let $M$ be the mixture operator, such that $\bar{Y}_t = M(Y_t)$ is the superimposed signal. 

\begin{theorem}
    Consider an observation map $O^{clean}_L$ when all $k$ signal sources $s_i$ are separately observable, and a vector of clean signal observations $Y_{0:L}= \left(y_{0:L}^{(1)}, ..., y_{0:L}^{(k)}\right)$. There exist environments $E$ and non-injective mixture operators $M$ on $z_1$ and $z_2$ in $\mathcal{Z}$ for which $\delta_{clean}(z_1, z_2) < \infty$, such that the superposition observation $\bar{Y}_{0:L} = M \circ O^{clean}_L(z_1) = M \circ O^{clean}_L(z_2)$ for all $L$. It follows that $\delta_{mixture}(z_1, z_2) = \infty$.
    \label{theorem:noninjective}
\end{theorem}

\begin{proof}
    The theorem is proved by example. Consider an environment where the incoming observations can be described by a $d$-dimensional binary vector $Y_t \in \{0, 1\}^d$ $d > 2$. Let there be three observations $y_{1...3}$: $y_1$ is zero everywhere except at index $i$, $y_2$ is zero everywhere except at index $j$ with $i \neq j$, and $y_3$ is zero everywhere except at $i$ and $j$. $y_1$ and $y_2$ correspond to the latent state $z_{1}$, and $y_3$ corresponds to $z_2$ with  $z_{1, 2} \in \mathcal{Z}$. It is assumed that observed signals $y_{1, ..., 3}$ are generated by the signal sources $s_{1, ..., 3}$. If no signal is generated by a source $s_i$, it emits the zero vector $\mathbf{0}$ that only contains zeros. Note that $z_1$ contains two signal sources, whereas $z_2$ contains only one. Assume first no signal mixture, and every signal is individually observable with $Y_t = (s_1, s_2, s_3)$. When the latent state is $z_1$, the observed signal is $Y_{t}^{(1)} = (s_1=y_1, s_2=y_2, s_3=\mathbf{0})$, whereas for $z_2$, the observation is $Y_{t}^{(2)} = (s_1=\mathbf{0}, s_2=\mathbf{0}, s_3=y_3)$. As every signal is cleanly observable, $Y_{t}^{(1)} \neq Y_{t}^{(2)}$ and $\delta(z_1, z_2) = 0 < \infty$.
    
    Now, set the mixture operator $M$ to be the arithmetic addition over the passed signals. When the latent state is $z_1$, the observed signal is $\bar{Y}_{t}^{(1)} = M \circ Y_t(s_1=y_1, s_2=y_2, s_3=\mathbf{0})$, which is 1 at $i$ and $j$ and zero otherwise. Similarly, for $z_2$, the agent observes $\bar{Y}_{t}^{(2)} = M \circ Y_t(s_1=\mathbf{0}, s_2=\mathbf{0}, s_3=y_3)$, which is, by definition, 1 at $i$ and $j$ and zero otherwise. It follows that $M$ is non-injective, as $\bar{Y}_{t}^{(1)} = \bar{Y}_{t}^{(2)}$ for all $t$, and $\delta(z_1, z_2) = \infty$. 
\end{proof}

It might be argued that the observation $\bar{Y}_{t}^{(1)} = \bar{Y}_{t}^{(2)}$ is never made or latent states $z_1, z_2$ are never reached. However, it is easy to construct an environment where the entire state space $\mathcal{Z} = \{z_1, z_2\}$ consists only of these two states. In this case, it is evident that $R$ goes to infinity, therefore having direct safety implications when the agent is used in an environment where an observational mixture operator $M$ might be present

\par Note that the superposition is a function of the observable signal sources in the environment, and it is therefore dependent on what the agent can perceive. It does not rely on a specific AI model. Moreover, non-identifiability can arise even under linear superposition. It follows that there can be no $R$-safety guarantee for any latency-$L$ controller in environments with signal mixture for which the superposition leads to indistinguishable observations of states with disjoint safe actions.

\section{Discussion}
This paper proves that general guarantees of $R$-safe systems for latency-$L$ controllers are incomplete when the environment, observational model, and safe actions are not explicitly considered.  Importantly, our proofs rely on the observational model of the environment, and it is independent of the choice of agent model or training. The following section discusses the results of this paper in context of the literature.

\par The superposition problem is widely known in various fields, including speech recognition \cite{cherry1953some}, radar sensing \cite{wiley2006elint, gunn2025radar}, and genetics \cite{zeitler2024next}. It therefore represents a realistic problem for autonomous agents. Theorem \ref{theorem:noninjective} proves that $R$-safety cannot be guaranteed, completely independent of the agent, its design, and training.  

\par The results presented here are motivated by Gold's language identification in the limit \cite{gold1967language}. Gold defines collapsing uncertainty over discrete time observations. Moreover, the theoretical bounds in both works are dependent on the structure of the environment and the observational model rather than the agent itself. However, the results presented here address a different problem. The $R$-safety results require an action by an agent within an $L$-bound time horizon based on an incompletely-observable latent state $z$, and observations $Y_t$ from multiple signal sources can be superimposed through a mixture operator $M$. In contrast, \cite{gold1967language} uses language theory to prove that, with an infinite time horizon as well as clean and separate observations, negative instances provided by an informant are required to identify any super-finite class of languages, unless a restrictive, computable ordering is imposed.

The development of our safety criteria builds on the theory of partially observable Markov decision processes (POMDPs) \cite{kaelbling1998planning}. Here, an agent cannot observe the environmental state directly and distinct states may produce identical observations, a phenomenon known as perceptual aliasing \cite{chrisman1992reinforcement}. \cite{kaelbling1998planning} show that a memoryless policy must act identically in all indistinguishable states, whereas maintaining a belief state over the agent history allows such states to be disambiguated over time. Our framework formalises the limits of this disambiguation under a maximum latency $L$. The set $V_\lambda$ contains the states that remain aliased after the full observation history up to $\lambda$, and the ambiguity distance $\delta(z,z')$ is equivalent to the history length required to resolve the aliasing between $z$ and $z'$. However, in our setup, the agent must act within $L$ observations, so aliasing that is resolved only after $L$ is equivalent to aliasing that is never resolved. Theorem \ref{theorem:infinite} shows that, in infinite state spaces, every aliased pair may eventually be resolved without a uniform bound on history length, and Theorem \ref{theorem:noninjective} shows that signal mixture can make aliasing permanent, such that no memory or belief tracking can disambiguate the states.

\par The formal construction presented here draws on the theory of infinite games with bounded $L$-delay \cite{fridman2011degrees, holtmann2012degrees, zimmermann2015unbounded}. More precisely, when defining an input and output player $A$ and $B$, an infinite game can be seen as a bit sequence $a_0b_0a_1b_1...$, where the individual sequences $a_0a_1... \in \Sigma_A^\omega$ and $b_0b_1... \in \Sigma_B^\omega$ are infinite words, and the winning condition can be recognised by a parity automaton. \cite{holtmann2012degrees} investigated $\omega$-regular winning conditions and proved that the lookahead required for a winning strategy is bounded double-exponentially in the size of the parity automaton. \cite{fridman2011degrees} showed that context-free winning conditions can require unbound lookahead. In contrast, Theorem \ref{theorem:infinite} presented in this paper proves that a finite lookahead horizon can break down when considering an infinite environmental state space, even if every pair of latent states $z, z' \in \mathcal{Z}$ have finite non-identifiablility. Whilst also demonstrating infinite bounds, it does not draw on the formal theory of infinite games, and it models real-world conditions which are only partially observable and which can therefore induce a non-identifiability problem. It is independent of the formulation of the winning condition, and the safe space must be explicitly structured to account for unbounded $R$. Furthermore, Theorems \ref{theorem:noninjective} proves that non-injective mixture operators $M$ can generate infinitely-long observational indistinguishability even when $|\mathcal{Z}| < \infty$. Establishing a finite upper bound is impossible. Signal superposition has not been addressed in the context of infinite games.

\par Guaranteed safe AI was previously proposed as a combination of a world model, a set or distribution of safety specifications, and a verifier that certifies, relative to the world model, that an AI system satisfies the specification \cite{dalrymple2024towards}. Specifically, the authors highlight that the world model describes how the AI system affects its environment and must be derived from data rather than only from physical laws. Our results are therefore consistent with the framework: if the world model contains states $z, z'$ with $O_L(z) = O_L(z')$ and $\bigcap_{z \in V_L} \mathcal{S}(z) = \emptyset$, then no latency-$L$ policy satisfies the specification. Our work extends the guaranteed-safe AI framework and makes it agent-independent. By imposing a time horizon $L$ and separating observational map $O_L$ from latent state $z$, $R$ determines whether a certificate can exist at latency $L$. Theorems \ref{theorem:infinite} and \ref{theorem:noninjective} identify two conditions that are independent of the agent and for which $R = \infty$. In the context of guaranteed safe AI, that means that no choice of world model, verifier, or agent can close the gap.


\par \cite{carr2023safe} integrate safe actions, which they refer to as shields, with deep reinforcement learning in POMDPs. They demonstrated empirically that specifically considering safe actions during reinforcement training can improve convergence and final safe performance. Moreover, their reinforcement environment allows for incorporating limited time horizons in the reward function. The definition of belief support and shields in \cite{carr2023safe} is equivalent to the definition of $\mathcal{V}_L$ and $\mathcal{S}$ used here, and the authors highlight that safety specification that define actions or situations to be avoided can lead to a block. We formulate this more generally as $\bigcap_{z\in V_L}\mathcal{S}(z)=\emptyset$. We highlight their work focuses on safety and convergence during reinforcement learning. Our results address the complementary question of whether a non-blocking shield can theoretically exist within latency $L$, independently of the agent. More specifically Theorem \ref{theorem:infinite} shows that for infinite state spaces the required number of observations may have $L$ that bounds all ambiguity horizons, even though every pair of states is eventually distinguished. Theorem \ref{theorem:noninjective} shows that signal mixture can merge conflict pairs such that every shield blocks all actions for all $L$.

\par Theorem \ref{theorem:infinite} has been stated in a similar way for infinite-dimensional environmental (world) vectors \cite{yao2025alignment}. However, the present paper puts the emphasis on the number of states rather than the dimensionality of the vector. Moreover, the paradox cited in \cite{yao2025alignment} argues that the rules required for analysing a complex environment cannot be enumerated, from which follows that alignment cannot be guaranteed, and it remains oblivious of environments were this can be done. The present paper makes the stronger argument that $R$-safety cannot be guaranteed as a function of the environmental latent space and the observations, and it is independent of the rule set, agent, or the task.

\section{Conclusions}
Instead of focusing on the AI agent, this paper formally proves that safety for autonomous systems, that need to act within a limited time horizon can only be guaranteed when considering explicitly the partially-observable environment together with the safe states. In particular, non-injective signal mixture of independent signal sources can require infinite lookahead. The paper contributes to the formal assessment of AI systems and their alignment under real-world conditions.

\section{Acknowledgments}
This work was supported by the Turing’s Defence and Security programme through a partnership with the UK government in accordance with the framework agreement between HMG and The Alan Turing Institute. 
Claude Opus 5 and ChatGPT-5.6 Luna were used for developing the framework. Figure 1 was generated using Claude Opus 5. 
We thank Varun Chhabra for their insightful comments and corrections.

\bibliography{main}
\end{document}